\documentclass{article}

    \PassOptionsToPackage{numbers, compress}{natbib}
 
\usepackage[preprint]{neurips_2024}

\usepackage{graphicx}
\usepackage[utf8]{inputenc} 
\usepackage[T1]{fontenc}    
\usepackage{hyperref}       
\usepackage{url}            
\usepackage{booktabs}       
\usepackage{amsfonts}       
\usepackage{nicefrac}       
\usepackage{microtype}      
\usepackage{xcolor}         
\usepackage{multirow}
\usepackage{amsmath}
\usepackage{amsthm}
\usepackage[english]{babel}

\newtheorem{theorem}{Theorem}
\newtheorem{definition}{Definition}

\usepackage{indentfirst}
\usepackage{comment}

\usepackage{tabularx}
\usepackage{amsmath}
\usepackage{cleveref}
\usepackage{amssymb}
\usepackage{MnSymbol}
\usepackage{braket}
\usepackage{bm} 
\title{3D-RPE: Enhancing Long-Context Modeling Through 3D Rotary Position Encoding}

\author{
     {Xindian Ma}\textsuperscript{\rm 1}, 
     {Wenyuan Liu}\textsuperscript{\rm 1},
     {Peng Zhang}{\textsuperscript{\rm 1}}{\thanks{Corresponding Author: Peng Zhang}}~~, 
     {Nan Xu}{\textsuperscript{\rm 2}}\\
     \textsuperscript{\rm 1} College of Intelligence and Computing, Tianjin University, Tianjin, China \\
     \textsuperscript{\rm 2} Beijing Wenge Technology Co.
Ltd.\\
     \{xindianma, lwy2020, pzhang\}@tju.edu.cn\\
     \{xunan2015\}@ia.ac.cn \\
}

\begin{document}

\maketitle
\begin{abstract}
Inspired by the Bloch Sphere representation, we propose a novel rotary position encoding on a three-dimensional sphere, named 3D Rotary Position Encoding (3D-RPE). 3D-RPE is an advanced version of the widely used 2D Rotary Position Encoding (RoPE), with two major advantages for modeling long contexts: controllable long-term decay and improved position resolution.
For controllable long-term decay, 3D-RPE allows for the regulation of long-term decay within the chunk size, ensuring the modeling of relative positional information 
between tokens at a distant relative position.
For enhanced position resolution, 3D-RPE can mitigate the degradation of position resolution caused by position interpolation on RoPE. 
We have conducted experiments on long-context Natural Language Understanding (NLU) and long-sequence Language Modeling (LM) tasks. From the experimental results, 3D-RPE achieved performance improvements over RoPE, especially in long-context NLU tasks.
\end{abstract}

\begin{figure}[ht]
\vskip 0.1in
\small
\begin{center}
\includegraphics[width=12.3cm,height=9cm]{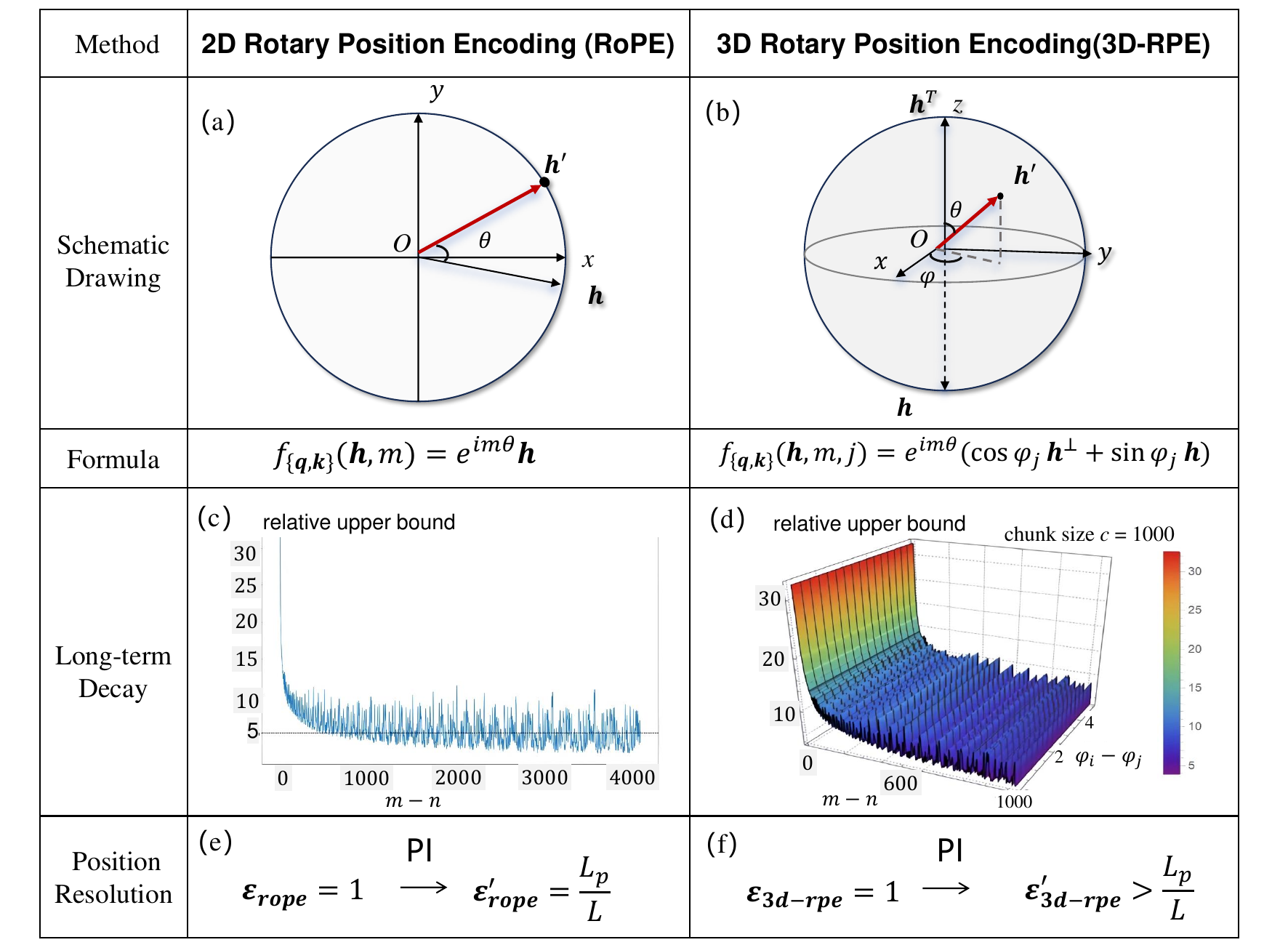}
\caption{2D Rotary Position Encoding (RoPE) vs. 3D Rotary Position Encoding (3D-RPE).}
\label{fig:correlation}
\end{center}
\vskip -0.2in
\end{figure}

\section{Introduction}
\label{introduction}
Rotary Position Encoding (RoPE)~\cite{su2024roformer} is essential in Transformer-based Large Language Models (LLMs), such as the LLaMA models~\cite{touvron2023llama}. 
RoPE merges the advantages of absolute and relative positional encoding by using a rotation mechanism to represent each position.
Despite its widespread use in LLMs~\cite{touvron2023llama,wang2021gptj6b,chiang2023vicuna}, RoPE has notable limitations when extending LLMs with a predefined context window. 
The long-term decay problem of RoPE limits the model's ability to extend positions outward in long-context tasks. 
Although the long-context modeling capability of LLMs can be extended through position interpolation, as more positions are inserted, RoPE encounters the challenge of decreased position resolution. 

We propose a novel position encoding mechanism for transformer architecture, called 3D Rotary Position Encoding (3D-RPE), to address challenges in long-context modeling faced by LLMs using RoPE.
Inspired by the Bloch Sphere representation, 3D-RPE applies rotary position encoding on a three-dimensional spherical surface, as illustrated in Figure~\ref{fig:correlation}(b). In contrast, RoPE employs rotation on a two-dimensional circular path, as depicted in Figure~\ref{fig:correlation}(a). 
RoPE suffers from long-term decay, as shown in Figure~\ref{fig:correlation}(c), implying that as the relative distance increases, the relative upper bound on token correlations at modeled relative positions will continuously decrease. 3D-RPE addresses this issue by segmenting a long sequence into chunks and setting rotation angles within and between the chunks to construct position encoding.
As shown in Figure \ref{fig:correlation}(d), 3D-RPE is able to control this relative upper bound through two relative positional dimensions, namely within and between chunks. Compared to Figure \ref{fig:correlation}(c), this method improves the upper bound on correlations between long relative distances and alleviates the issue of long-term decay.

Position Interpolation (PI) methods~\cite{chen2023extending, peng2023yarn} based on RoPE are often employed to extend LLMs for modeling contexts that exceed the pre-training length.
These techniques scale the position encoding during inference, allowing the originally out-of-range position encoding to fall within the trained position interval after interpolation.
However, as the interpolation factor increases, PI experiences a substantial decline in positional resolution among tokens, detrimentally affecting long-context modeling performance. 
As illustrated in Figure~\ref{fig:correlation}(e), extending the pre-training length $L_p$ to $L$ using linear PI~\cite{chen2023extending} leads to reduced positional resolution with increasing $L$. 
3D-RPE employs a 3D rotating sphere for position encoding, which supports higher positional resolution compared to the 2D circular rotation. Similarly, through linear PI extension, 3D-RPE achieves a positional resolution superior to $\frac{L_p}{L}$ (See Figure~\ref{fig:correlation}(f)).
This benefit has been theoretically proven (Refer to Theorem 1 in Section~\ref{sencond:benefits}) and corroborated by experimental results (Refer to Table~\ref{train_free_valid_pg19} in Section~\ref{experiments:nlu}).

We conducted experiments on long-sequence Language Modeling (LM) and long-context Natural Language Understanding (NLU) tasks. Our experimental results highlight the promising performance of the 3D-RPE method, especially in tasks requiring long-context language understanding.

Our major contributions of this paper are as follows:
\begin{itemize}  
  \item A position encoding method on a 3D sphere, 3D-RPE, is provided, which can enhance the long-context modeling capability of LLMs by replacing RoPE.
  \item It is proved that 3D-RPE has two benefits, controllable long-term decay and mitigating the reduction in positional resolution caused by position interpolation.
  \item LLMs combine with 3D-RPE have achieved significant performance improvements in long-context NLU tasks.
\end{itemize}
The structure of this paper is as follows. Section~\ref{Preliminaries} covers the preliminaries of 3D-RPE,
Bloch Sphere, and RoPE. Section~\ref{method} explains the construction of 3D-RPE on a 3D rotating sphere and highlights its benefits over RoPE. Section~\ref{related_work} reviews related work. In Section~\ref{experiments:all}, we validated the advantages of our method through experiments. Section~\ref{couclusion} concludes with a discussion on 3D-RPE's impact.

\section{Preliminaries}
\label{Preliminaries}
The analysis of 3D-RPE relies on these concepts and results from the filed of Bloch Sphere and RoPE. We offer an introduction to Bloch Sphere in Section~\ref{pre:bs} and  RoPE~\cite{su2024roformer} in Section~\ref{pre:rope}.
\subsection{Bloch Sphere}
\label{pre:bs}
Bloch Sphere (BS) offers a geometric depiction of a quantum mechanical system's pure state, limited to two levels. 
The state vector $|\phi\rangle$ is mathematically expressed as
\begin{equation}
\begin{aligned}
\ket{\phi}=e^{\mathrm{i}\theta}(\cos{\frac{\varphi}{2}}\ket{0}+\sin{\frac{\varphi}{2}}e^{i\theta_1}\ket{1})
\end{aligned}
\label{bs-equ}
\end{equation}
where $|0\rangle$ and $|1\rangle$ are Dirac's notations. ${\theta}$, ${\theta}_1$ and $\varphi$ are rotation angles.
In our work, $\theta$ encodes the relative positions of tokens within chunks, $\varphi$ encodes the relative positions of tokens across chunks, and $\theta_1$ is equal to $0$.
Some other concepts about BS are showed in Supplementary Materials A.

\subsection{Rotary Position Embedding}
\label{pre:rope}  
Rotary Position Embedding (RoPE) is a commonly used relative position encoding technique in LLMs, such as LLaMA~\cite{touvron2023llama}, GPT-J~\cite{wang2021gptj6b}, Vicuna~\cite{chiang2023vicuna} and etc. 
RoPE is a 2-dimensional space rotary encoding, which is denoted as follows:
\begin{equation}
\begin{aligned}
\label{eq:2d-rope}
RoPE(\bm{h}_m, m) = e^{\mathrm{i}m\theta}\bm{h}_m ~,~
RoPE(\bm{h}_n, n) = e^{\mathrm{i}n\theta}\bm{h}_n
\end{aligned}
\end{equation}
$\bm{h}_m$ and $\bm{h}_n$ are hidden vectors from the query and key for a specific attention head in transformer. For ease of differentiation, $\bm{h}_m$ and $\bm{h}_n$ can be refined later as $\bm{q}_m$ and $\bm{k}_n$,
$\mathrm{i}$ is the imaginary unit, $\theta$ is the rotary angle in RoPE. $m$ and $n$ are indexes about positions. 
Then, the inner product is employed to define the self-attention score before softmax computing:
\begin{equation}
\label{rope_attention}
\begin{aligned}
s(m-n,\bm{q}_m,\bm{k}_n) &=  \langle RoPE(\bm{q}_m, m), RoPE(\bm{k}_n, n) \rangle \\
       &= Re[\sum_{l=0}^{{d/2}-1}\bm{q}_{[2l:{2l+1}]}\bm{k}_{[2l:{2l+1}]}e^{\mathrm{i}(m-n)\theta_l}]
\end{aligned}
\end{equation}
Eq~(\ref{rope_attention}) is unary function respect to the relative position $(m-n)$, representing the relative position 
between tokens and modeling the relative positional information. Here,  $Re[\cdot]$ denotes the calculation of the real part of a complex number. In our study, the 3D-RPE self-attention score is a binary function containing the relative position $(m-n)$. 

\section{Method}
\label{method}
Section~\ref{cpe:method} introduces the new position encoding on a 3D sphere, 3D-RPE. Section~\ref{Properties} focuses on analyzing two benefits of 3D-RPE, namely controllable long-term decay and enhanced position resolution.

\begin{figure}[t]
\vskip 0.1in
\begin{center}
\includegraphics[width=12.0cm,height=9.5cm]{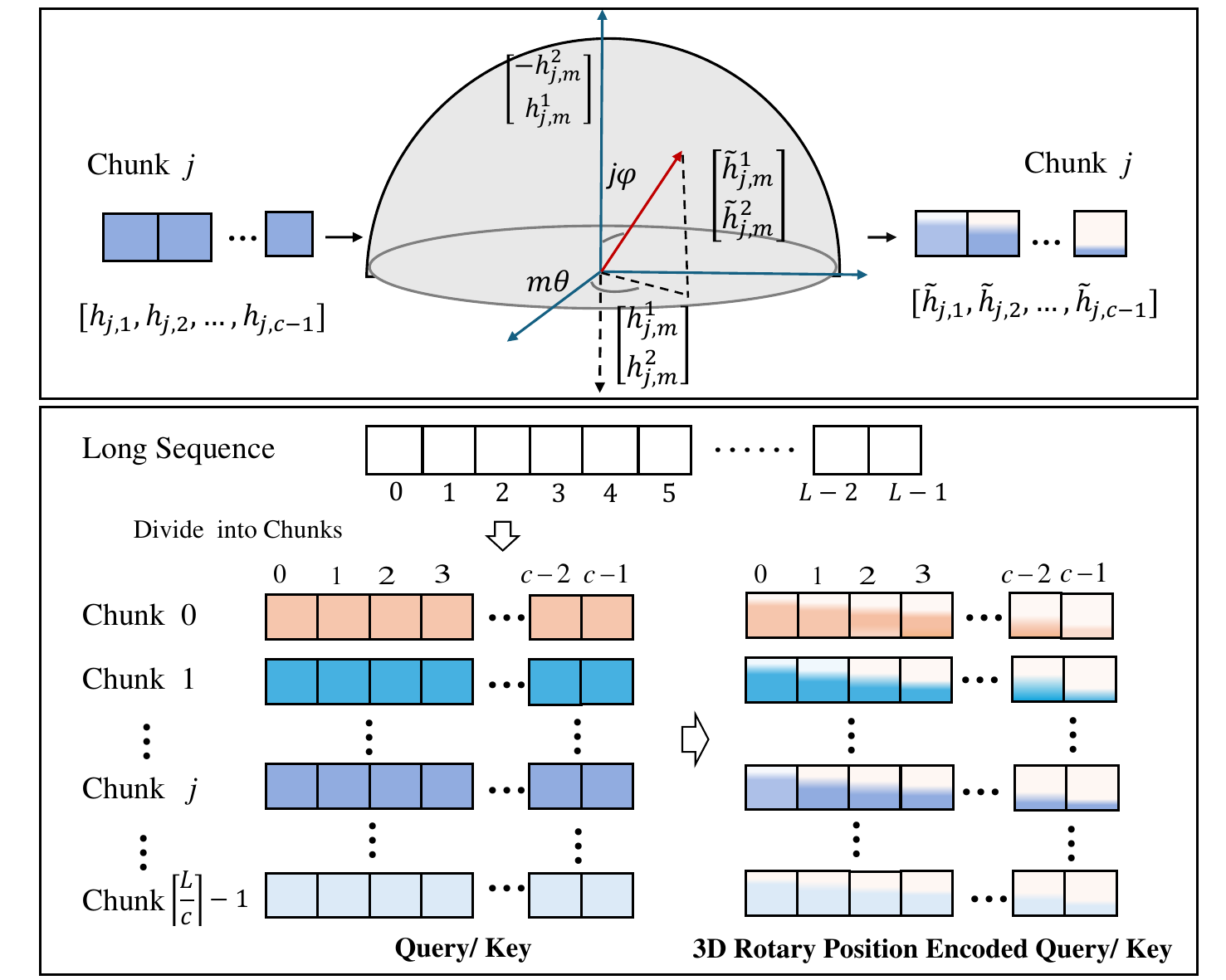}
\caption{
Visualization of the 3D Rotary Position Encoding (3D-RPE). The context size is $L$, and the chunk size is $c$. The vectors ${[\bm{h}_{j,m}^{1}, \bm{h}_{j,m}^{2}]}^{T}$ and ${[-\bm{h}_{j,m}^{2}, \bm{h}_{j,m}^{1}]}^{T}$ form an orthogonal basis, corresponding to the $\ket{1}$ and $\ket{0}$ states in Eq.~(\ref{bs-equ}). The components $\bm{h}_{j,m}^{1}$ and $\bm{h}_{j,m}^{2}$ represent the first and second dimensions of the state vector $\bm{h}_{j,m}$, which is the $m_{th}$ token in the $j_{th}$ chunk.
}
\label{cpe:model_struct}
\end{center}
\vskip -0.2in
\end{figure}

\subsection{3D Rotary Position Encoding}
\label{cpe:method}
For a long sequence of length $L$ and a chunk size set to $c$, where $c$ is smaller than the pre-training length of LLM, the sequence can be divided into $\lceil L/c\rceil$ chunks. 
Here, $\lceil . \rceil$  represents the ceiling function, rounding up to the nearest integer (see Figure~\ref{cpe:model_struct}). 
The state vector $\bm{h}_{j,m}$ comes from either Query or Key. Here, $j \in [0, {\lceil L/c\rceil}-1]$ represents the positional index of the chunk, and $m \in [0, c-1]$ indicates the positional index of the token within the chunk. This is used to calculate the new state vector $\widetilde{\bm{h}}_{j,m}$ by rotating the Bloch Sphere.
Specifically, two rotation angles, $\theta$ and $\varphi$ are defined, with $\theta$ governing the position encoding within the chunk’s internal tokens, and $\varphi$ governing the position encoding between the chunks. 
Our position encoding method is called 3D Rotary Position Encoding, or 3D-RPE.
The formal definition of 3D-RPE is provided as follows. The computational process of 3D-RPE in practice is provided in Supplementary Materials B.1. 
\begin{definition} [\textbf{3D Rotary Position Encoding}]
Let $\bm{h}_{j, m}\in\mathbb{R}^d$ be a state vector of an attention head without position encoding, where $d$ is the dimension of the vector, which is an even number. 3D-RPE encodes $\bm{h}_{j, m}$ into the vector $\widetilde{\bm{h}}_{j,m}$, which can be formalized as:
\begin{equation}
\label{eq:cpe-general}
\widetilde{\bm{h}}_{j,m} = e^{\mathrm{i}m\theta}(\cos{\varphi_j}\bm{h}_{j, m}^{\perp} + \sin{\varphi_j}\bm{h}_{j, m})
\end{equation} 
$\mathrm{i}$ is the imaginary unit.
$\bm{h}_{j, m}^{\perp}$ equals to ${[-\bm{h}_{j,m}^{2}, \bm{h}_{j,m}^{1}]}^{T}$, where $\bm{h}_{j,m}^{1} \in \mathbb{R}^{d/2}$ and $\bm{h}_{j,m}^{2} \in \mathbb{R}^{d/2}$ is the first and second halves of the state vector $\bm{h}_{j,m}$. 
\end{definition}

In transformer-based LLMs, after applying position encoding  to the state vectors from Query and Key, it is essential to compute their attention scores. 
For the sake of clarity and formalization, we denote the position encoding of the state vector from Query as 3d-PE$(\bm{q},i,m)$ and from Key as 3d-PE$(\bm{k},j,n)$, where $i$ and $j$ range from $0$ to ${\lceil L/c\rceil}-1$, and $m$ and $n$ range from $0$ to $c-1$. The self-attention score can be obtained through the conjugate symmetric inner product of $\bm{q}_{i,m}$ and $\bm{k}_{j,n}$, which are the state vectors from Query and Key, 
\begin{equation}
\label{eq:general-cpe-attn}
s(\bm{q}_{i,m},\bm{k}_{j,n},\varphi_i - \varphi_j,m-n)=Re[e^{\mathrm{i}(\varphi_i-\varphi_j)}\sum\limits_{l=0}^{d/2-1}e^{\mathrm{i} (m-n)\theta_l}(\bm{q}_{l}\bm{k}_{l}+\bm{q}_{d/2+l}\bm{k}_{d/2+l})]
\end{equation}
where $l \in [0, {\frac{d}{2}-1}]$, $\varphi_i=base^{-i}$ and $\varphi_j=base^{-j}$. Let ${\{\bm{q},\bm{k}\}}_l$ denote the $l$-th components of ${\{\bm{q},\bm{k}\}}$. 
In experiments using the LLaMA2 models, the $base$ is generally set to $10,000$.
The self-attention score computed after applying 3d-PE is a function of both the relative position between chunks ($\varphi_i- \varphi_j$) and the relative position ($m-n$). 

Consequently, the self-attention score relying on 3d-PE is influenced by the relative positions at both the chunk and token levels.
It is important to highlight that when $\bm{q}_{i,m}$ and $\bm{k}_{j,n}$ reside within the same chunk (i.e., $i=j$), Eq. (\ref{eq:general-cpe-attn}) simplifies to the standard RoPE formulation as depicted in Eq. (\ref{rope_attention}). 
For a detailed derivation and computation process of Eq. (\ref{eq:general-cpe-attn}), as well as the complete formulation of Eq. (\ref{eq:cpe-general}), please refer to Supplementary Materials B.2.


\subsection{Benefits of 3D-RPE}
\label{Properties}
In this section, we delve into two benefits offered by 3D-RPE: the ability to control long-term decay and mitigate the reduction in positional resolution caused by position interpolation.

\subsubsection{Controllable Long-term Decay}
3D-RPE has the property of controllable long-term decay. Like RoPE, taking the absolute value $s$ in Eq~(\ref{eq:general-cpe-attn}) and applying the Abel transformation, we derive the upper bound of the correlation coefficients related to term dependencies as follows:

\begin{equation}
\label{upper_bound_cpe}
\begin{aligned}
|s(\bm{q}_{i,m},\bm{k}_{j,n},\varphi_i - \varphi_j, m-n)|
&\leq|e^{\mathrm{i}(\varphi_i - \varphi_j)}||\sum\limits_{l=0}^{d/2-1}E_{l+1}(h_{l+1}-h_l)| \\
&\leq(\max_{l}|h_{l+1}-h_l|)\sum\limits_{l=0}^{d/2-1} |E_{l+1}|\end{aligned}
\end{equation}
where $E_l=\sum_{k=0}^{l-1}e^{\mathrm{i}(m-n)\theta_k}$ and $E_0=0$.
For RoPE~\cite{su2024roformer}, the relative upper bound $E_{rope}$ is given by
$\frac{1}{d/2} \sum_{j=1}^{d/2}|S_j|$, where $S_{j}=\sum_{t=0}^{j-1}e^{i(m-n)\theta_{t}}$ (see the section 3.4.3 of RoPE~\cite{su2024roformer}). By setting $\theta_t = 10000^{\frac{-2t}{d}}$, the value decays as the relative position $(m-n)$ increases. For the upper bound $E_{3d-rpe}$ of 3D-RPE, it is formalized as follows:
\begin{equation}
\label{upper_bound}
E_{3d-rpe} = \frac{1}{d/2} \sum_{j=1}^{d/2}|E_l|
\end{equation}
The domains of the relative position $(m-n)$ differ between $E_{3d-rpe}$ and $E_{rope}$. In $E_{rope}$, $(m-n)$  is in the range $[0,L-1]$, while in $E_{3d-rpe}$, it is in $[0,c-1]$. The relative positions between tokens exceeding the chunk size $c$ are constructed collaboratively using positional encoding within and across chunks.
The Relative Position Matrix $\bm{A}$ using 3D-RPE is shown in Figure~\ref{RelativePositionMatrix}.

To illustrate the advantage of controllable long-term decay, we present the results in Figure~\ref{fig:correlation}(c) and Figure~\ref{fig:correlation}(d). As shown in Figure~\ref{fig:correlation}(c), when the relative position $(m-n)$ exceeds approximately $1000$, $E_{rope}$ begins to significantly decrease to below $5$. 
This limitation of $E_{rope}\leq5$ poses challenges for RoPE in modeling attention scores between tokens with longer relative distances (greater than $4000$).
In contrast, as shown in Figure~\ref{fig:correlation}(d), 3D-RPE employs both $(m-n)$ and $(\varphi_i-\varphi_j)$, setting $c=1000$ to keep $(m-n)$ within $1000$, thereby preventing decay over longer distances. This method ensures $E_{3d-rpe}$ stays at or above $5$ for all relative positions.

\begin{figure}[t]
\vskip 0.1in
\begin{center}
\includegraphics[width=\columnwidth]{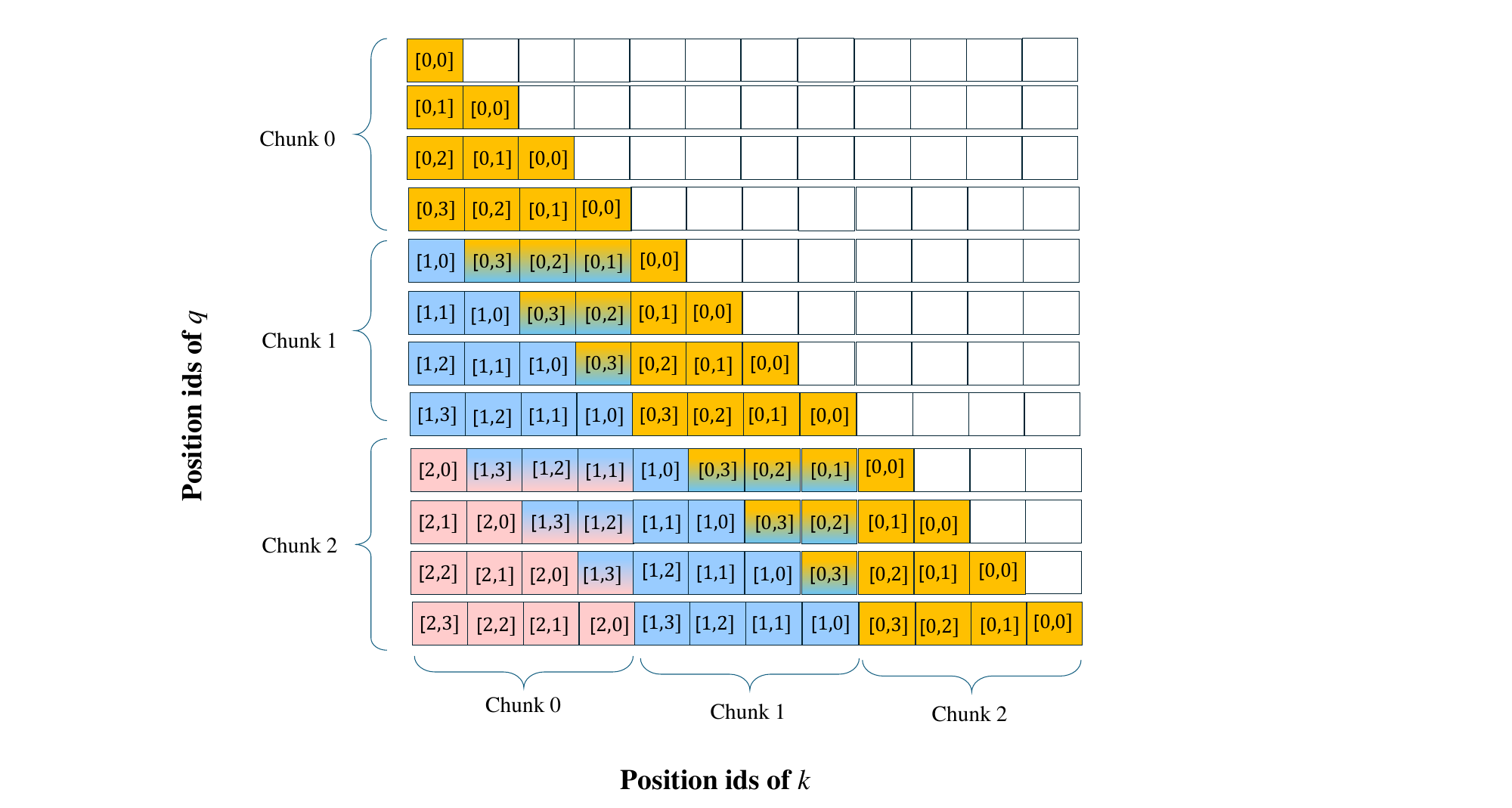}
\caption{Visualization of the Relative Position Matrix $\bm{A}$ employing 3D-RPE, with chunk size $c$=$4$, and sequence size $L$=$12$. The matrix elements $A_{i,j}$ represents the relative position between the $i_{th}$ query vector $\bm{q}$ and the $j_{th}$ key vector $\bm{k}$.}
\label{RelativePositionMatrix}
\end{center}
\vskip -0.2in
\end{figure}

\subsubsection{Enhanced Positional Resolution}
\label{sencond:benefits}
Position Interpolation (PI)~\cite{chen2023extending} has been introduced to scale down the position indices to align with the original window size, resulting in enhanced outcomes for context extension. 
However, as the extension length and interpolation increase, PI can lead to a reduction in relative positional resolution.
3D-RPE can be used alongside PI for long-context extensions. 
Compared to RoPE combined with PI, 3D-RPE has the advantage of mitigating the reduction in positional resolution caused by positional interpolation, as demonstrated in Theorem~\ref{thm:interpolation}.
\begin{theorem}
[\textbf{Enhanced Position Resolution}]
\label{thm:interpolation}
For a pre-trained language model with a length of $L_p$ and an extension length requirement of $L$, employing linear position interpolation extension methods $\mathcal{I}$ based on Rotary Position Encoding (RoPE) can elevate the relative positional resolution from $\mathcal{E}_{rope}$ to $\mathcal{E}_{rope}^\prime$. Let $\mathcal{E}_{3d-rpe}^\prime$ denote the relative positional encoding resolution achieved by the method $\mathcal{I}$ based on 3D-RPE, with chunk size $c \geq 3$, there is:
\begin{equation}
\label{cpe:resolution}
   \mathcal{E}_{3d-rpe}^\prime  > \mathcal{E}_{rope}^\prime
\end{equation}
\end{theorem}
The Proof of \textbf{Theorem 1} is provided in Supplementary Materials C. 

To empirically validate the superior performance of this benefit in a training-free setting, it has been observed that methods combining RoPE with interpolation lead to a significant increase in Perplexity as the modeling length increases in language modeling tasks. Conversely, the increase in Perplexity is substantially smaller when employing 3D-RPE with linear interpolation (Refer to Table~\ref{train_free_valid_pg19} in Section~\ref{experiments:all}). This phenomenon indicates that this benefit has led to an improvement in the performance of long sequence language modeling.

\section{Related Work}
\label{related_work}
This section provides an overview of the extensive literature on position encoding in Transformers~\cite{vaswani2017attention} and discusses context extending capabilities based on RoPE.

\textbf{Position Encoding (PE):} PE is important for Transformer-based language models. Earlier studies~\cite{shaw2018selfattention,raffel2020exploring,wang2020encoding,su2024roformer} have focused on enhancing the original absolute position encoding to develop better relative position encoding, thereby improving the text modeling capabilities of language models. 
These works~\cite{shaw2018selfattention,raffel2020exploring,wang2020encoding} utilized trainable position vector encoding to directly incorporate positional information into context representations. 
Although effective, these methods typically add positional information to contextual representations, making them unsuitable for linear self-attention architectures. 
RoFormer~\cite{su2024roformer} introduced relative position information by rotating context representations, known as RoPE. 
Transformers utilizing RoPE have become a prevalent backbone in various LLM designs~\cite{touvron2023llama,chowdhery2022palm,wang2021gptj6b,nijkamp2022codegen}. 
Our proposed 3D-RPE differs from the two-dimensional space of RoPE by modeling the relative position information of tokens through rotation on the Bloch Sphere.

\textbf{Long-context LLMs based on RoPE:} 
To enhance the contextual capabilities of Large Language Models (LLMs) using RoPE, several positional encoding interpolation techniques have been developed.
These include Linear Position Interpolation (LPI)~\cite{chen2023extending}, Neural Tangent Kernel (NTK)~\cite{peng2023ntk}, and Yet Another Recurrent Network (YaRN)~\cite{peng2023yarn} interpolation. 
Position Sequence Tuning (PoSE)~\cite{zhu2023pose} has notably increased sequence lengths to $128k$ by amalgamating these positional interpolation strategies. 
Additionally, LongLora~\cite{longlora} introduced the shift-short attention mechanism, allowing for effective emulation of full attention and extending sequences up to $100k$, leveraging the LLMa-2-7B model and LoRA's fine-tuning approach~\cite{hu2022lora}.
3D-RPE further strengthens the positional relationships between distant tokens by capturing inter-chunk positional information and is compatible with existing fine-tuning techniques like LoRA to bolster long-context representation.
The Dual Chunk Attention (DCA)~\cite{an2024trainingfree} method, which enhances the use of pre-trained integer-based parameters, splits query and key sequences into chunks and uses three specialized matrices to capture the relative positions within and between these chunks. This method enhances the model’s ability to process longer sequences, but it is unable to model the relative positions within distant chunks.
In our work, we employ rotating positional encoding to link attention across different chunks.

\section{Experiments}
\label{experiments:all}
We evaluate the method of position encoding, 3D-RPE, on LLaMA2~\cite{touvron2023llama} models (specifically, LLaMA-2-7B and LLaMA-2-7B-chat), which have a $4k$ pre-training context, and LLaMA-3-8B-Instruct~\footnote{https://github.com/meta-llama/llama3}, which has an $8k$ pre-training context. Our experiments aim to explore the following aspects: 1) The effect of 3D-RPE on long-context generation can be assessed using Perplexity.
2) The impact of 3D-RPE on long-context understanding and generation tasks, can be reflected by the accuracy of long sequence natural language tasks, e.g., multiply documents QA. 
3) Ablation studies to confirm the advantages of 3D-RPE in position interpolation.

\subsection{Experimental Settings}
In this section, we elaborate on the experimental setup by introducing two types of tasks (i.e., long-context language understanding and long sequence language modeling) and detailing three aspects of the configuration (i.e., training parameters, training and evaluation datasets, and baseline models).

\textbf{Training Setting:} 
For long-context Natural Language Understanding (NLU) tasks, we have fine-tuned LLaMA-2-7B-chat and LLaMA-3-8B-Instruct. The context length for these models has been extended from $4k$ to $16k$ and from $8k$ to $16k$, respectively.
The fine-tuning method follows the fine-tuning strategy of LongChat~\cite{longchat2023}. The training step is $3,000$.
For the long-sequence Language Modeling (LM) tasks, we have fine-tuned LLaMA-2-7B to support extended context length of $32k$ tokens. The training step is $1,000$.
We set the per-device batch size as $1$, and gradient accumulation step as $8$, which means that the batch size is $8$. We train the model with the next token prediction objective with LoRA~\cite{hu2022lora}. 

We employed the AdamW optimizer~\cite{AdamW} with ${\beta}_1 = 0.9$ and ${\beta}_2 = 0.95$ for all fine-tuned models. Chunk size is set to $3k$.
The learning rate was set to $2\times 10^{-5}$, and a linear learning rate warmup was applied. Training was conducted on a single 4xA800 GPU machine using FlashAttention-2~\cite{Dao2023}.

\textbf{Datasets:} 
In the context of long-context NLU tasks, we employ the LongAlpaca-12k dataset, which contains 9,000 LongQA and 3,000 short QA entries~\cite{long-alpaca}, and the LongAlpace-16k-length dataset\footnote{https://github.com/dvlab-research/LongLoRA/}.
To evaluate the performance of 3D-RPE for long-context extension, we use the LongBench~\cite{bai2023longbench}, which includes $13$ English tasks, $5$ Chinese tasks and $2$ code tasks, with most tasks having an average context length of $5k$ to $15k$ tokens. We focus on the English and code tasks to evaluate our method, 3D-RPE. 
Additionally, the LEval~\cite{an2023eval} evaluation set, which also consists of long-context datasets, is used to verify the effectiveness of 3D-RPE. The five datasets annotated from scratch in LEval,
namely Coursera, QuALiTY, CodeU, GSM,and TOEFL, are utilized.

For long-sequence LM tasks, we use the RedPajama-Data~\cite{redpajama} for fine-tuning training.
The dataset is a large-scale pre-training dataset (the size reaches 1.2 trillion tokens) designed to provide high-quality training data for language models, and contains multiple data sources (i.e., github, arxiv, book, c4 and Wikipedia, etc.). We sample $20,000$ samples from these data sources for training. For evaluation, we utilize the PG19 book corpus dataset~\cite{rae2020compressive}, which includes 100 documents, and the Arxiv Math Proof-pile dataset (test split). Additionally, all methods evaluate perplexity by using a sliding window following~\cite{2021arXiv210812409P}.

\textbf{Baseline models:} 
For long-context NLU tasks, the fine-tuned models, including LongAlpace-16k~\cite{longlora}, LongChat-32k~\cite{li2023how} LongLlama~\cite{Tworkowski2023FocusedTransformer} and ChatGLM~\cite{du2022glm} are used as the baseline models. 
Models of fine-tuning free in language modeling tasks are also used in long-context NLU tasks. 

In long sequence LM tasks, the methods of LongLoRA~\cite{longlora}, StreamingLLM~\cite{xiao2023streamingllm}, Positional Interpolation(PI)~\cite{chen2023extending} and the NTK-Aware Scale RoPE(NTK)~\cite{peng2023ntk} are selected as the baselines, all based on the LLaMA-2-7B-base model. Among these baseline models, PI, NTK and StreamingLLM are fine-tuning-free methods. The fine-tuned models include LongLoRA and Activation Beacon~\cite{zhang2024soaring}.
In Ablation experiments, interpolation methods without training are used as baseline models, which are PI and NTK.

\begin{table} [t]
\caption{Comparison between open-source based models on long-context NLU tasks.
Our model, 3D-RPE-LlaMA2-7B-Chat is fine-tuning on LLaMA-2-7b-chat, which is extended from $4k$ to $16k$ context lengths.
Baseline models can be categorized into two groups: those that necessitate fine-tuning during training (such as LongAlpaca~\cite{longlora} and LongLLaMA~\cite{Tworkowski2023FocusedTransformer}), and those that do not require it (including PI, NTK, StreamingLLM~\cite{xiao2023streamingllm}, and ChunkLLaMA-$16k$~\cite{an2024trainingfree}).
The experimental results for each specific evaluation set in Supplementary Material D.2.} 
\label{nlu-task-table}
\begin{center}
\setlength{\tabcolsep}{2pt} 
\renewcommand{\arraystretch}{1.1} 
\begin{small}
\begin{sc}
\begin{tabular}{lccccc} \hline
 Methods & \textnormal{Single-Doc QA} & \textnormal{Multi-Doc QA} & \textnormal{Summarization} & \textnormal{Few-shot} & \textnormal{Code}  \\ 
\hline
\textnormal{LLaMA-2-7B-chat} &     24.90    &    22.60       &      24.70      &     60.01  &   48.10    \\ 
\textnormal{LLaMA-2-7B-chat-PI}       & 18.98        &    17.16       &      25.03      &     49.43  &   52.73    \\
\textnormal{LLaMA-2-7B-chat-NTK}             & 23.21        &    23.34       &      24.40      &     59.29  &   49.28    \\ 
\textnormal{StreamingLLM}    & 21.47      &    22.22       &      22.20      &     50.05  &   48.00   \\
\textnormal{ChunkLLaMA-$16k$}    &    24.04     &    22.98          &      21.52         &     46.31     &   49.73    \\
\hline
\textnormal{LongChat-$32k$}   &    31.58   &    23.50       &      26.70      &     64.02  &    54.10 \\
\textnormal{LongAlpaca-$16k$}  &   28.70   &    28.10       &      27.80      &     63.70  &    56.00 \\
\textnormal{LongLLaMA}      &   30.12   &    16.37       &      24.19      &     60.31  &    66.05 \\ 
\textnormal{Vicuna-v1.5-7B-$16k$} & 28.01 & 18.63 & 26.01 & 66.20 & 47.30  \\
\textnormal{ChatGLM3-6B-$32k$} & 40.30 & 46.60 & \textbf{29.50} & 68.10  & 56.20 \\
\hline
\textnormal{3D-RPE-LLaMA2-7B-Chat} &  \textbf{47.40}  &  \textbf{60.10}  &  28.99  &    \textbf{73.16} &  \textbf{76.50}  \\
\hline
\end{tabular}
\end{sc}
\end{small}
\end{center}
\vskip -0.1in
\end{table}

\begin{table}[t]
\caption{Comparison with open-source models, LLaMA-2-7B-chat, LLaMA3-8B-Instruct, on 5 closed-ended-ended tasks with various input length from LEval~\cite{an2023eval}. The evaluation metric “EM,” which represents the exact match score, is adopted.
* indicates the model is train-free.}
\label{longeval:leval}
\begin{center}
\begin{sc}
\begin{tabular}{lccccccc} \hline
\toprule
{Models} & \textnormal{Coursera} &  \textnormal{QuALiTY} & \textnormal{CodeU} & \textnormal{GSM} & \textnormal{TOEFL} \\
\midrule 
\textnormal{LLaMA-2-7B-Chat} & 29.21 &  37.62 & 1.11 & 19.00 & 51.67 \\
\textnormal{LongChat-7B-16K} & 29.74 &  33.66 & 3.33 & 10.00 & 47.95 \\
\textnormal{LLaMA2-7B-NTK}  &  32.71 &  33.16 & 0.00 & 19.00 & 52.78  \\
\textnormal{Vicuna1.5-7B-16k} &  38.66 & \textbf{39.60} & \textbf{5.55} &  19.00 & 55.39 \\
\textnormal{3D-RPE-LLaMA2-7B-Chat(ours)} & \textbf{39.38} & 38.11 & 2.22 & \textbf{21.01}  &  \textbf{57.99} \\
\hline 
\textnormal{LLaMA3-8B-Instruct*}  & 51.45 & \textbf{64.34} & 4.44 & 76.00  & 82.89 \\
\textnormal{3D-RPE-LLaMA3-8B-Instruct*} & \textbf{51.89} & 61.38 & 4.44 & \textbf{80.00} & 82.89 \\
\bottomrule
\end{tabular}
\end{sc}
\end{center}
\vskip -0.2in
\end{table}

\begin{table}[t]
\caption{Perplexity evaluation on different extending methods. We conduct evaluation on the Proof-pile and PG-19 test datasets, varying evaluation context window size from $8k$ to $100k$.}
\label{lm-task-table}
\begin{center}
\renewcommand{\arraystretch}{1.1} 
\begin{small}
\begin{sc}
\begin{tabular}{lcccccccc} \hline
\multicolumn{1}{c}{\multirow{2}{*}{Methods}} & \multicolumn{4}{c}{\textnormal{PG-19}} & \multicolumn{4}{c}{\textnormal{Proof-Pile}} \\ \cline{2-9} 
& $8k$  & $16k$ & $32k$ & $100k$ &  $8k$ & $16k$ & $32k$ & $100k$ \\ 
\midrule
\textnormal{LLaMA2-7B-Base} & 131.09 & $>10^2$ & $>10^2$ & OOM  & 16.79 &$>10^2$ & $>10^2$  & OOM \\
\textnormal{LLama2-7B-PI}   & 11.32  & 19.5    &$>10^2$  & OOM    &  3.86  & 5.94   & 33.7    & OOM  \\
\textnormal{LLama2-7B-NTK}        & 10.28  & 11.5    & 37.8    & OOM  &  3.98  & 5.94    & 33.7   & OOM  \\ 
\textnormal{StreamingLLM} & 9.23   & 9.25    & 9.24    & 9.32  & 3.47  & 3.51   & 3.50    & 3.55 \\ 
\hline 
\textnormal{LongLoRA-32k}  & 7.33   & 7.16    & \textbf{7.04}  & {--}  & 2.78 & 2.61 & \textbf{2.50} & {--}  \\
\textnormal{LongLoRA-100k}  & 7.57   & 7.33    & 7.16  & \textbf{7.04} & 2.78 & \textbf{2.60} & 2.58 & \textbf{2.52}  \\
\textnormal{LongChat-32k} & 8.92 &  8.85  &  8.81 & OOM & 2.98 & 2.70 & 2.65 & OOM \\ 
\textnormal{Activation Beacon} & 8.52 &  8.54  &  8.56 & 8.68 & 3.45 & 3.42 & 3.39 & 3.35 \\ 
\hline
\textnormal{3D-RPE-LLaMA2-7B} &  \textbf{7.03}  & \textbf{7.10}  & 8.09 & 8.12   & \textbf{2.72}  & 2.93  & 2.89  & 3.05 \\
\hline
\end{tabular}
\end{sc}
\end{small}
\end{center}
\vskip -0.1in
\end{table}

\begin{table}[t]
\caption{Results are evaluated in Perplexity on PG19 validation split. '*' denotes train-free, implementing 3D-RPE directly on the LLaMA2-7B-Base model without additional fine-tuning, utilizing a chunk size of $3k$.  The context length of $8k$ is extended directly with 3D-RPE. Achieving $16k$ and $32k$ is accomplished through linear positional interpolation with chunks based on the $8k$ context length.
}
\label{train_free_valid_pg19}
\begin{center}
\begin{sc}
\begin{tabular}{lcccccc} \hline
\toprule
{Models} & $4k$ & $8k$ & $16k$ & $32k$ \\
\midrule 
\textnormal{LLaMA2-7B-PI}       & 7.94   &  9.19    & 15.11  & $>10^2$ \\
\textnormal{LLaMA2-7B-NTK}      & 7.87   &  11.98   & 26.12  & 58.91  \\
\textnormal{LLaMA2-7B-Yarn}     & 7.87   &  8.06    & 9.82   & 11.74  \\
\midrule  
\textnormal{3D-RPE-LLaMA2-7B*} & 7.87   & \textbf{7.90}  & \textbf{7.71}  & \textbf{9.34}  \\
\bottomrule
\end{tabular}
\end{sc}
\end{center}
\vskip -0.2in
\end{table}

\subsection{Long-Context Natural Language Understanding}
\label{experiments:nlu}
In this task, the LongBench~\cite{bai2023longbench} evaluation set was initially utilized. Five categories of tasks were included: single-document QA (3 tasks), multi-document QA (3 tasks), summarization (3 tasks), few-shot learning (3 tasks), and code completion (2 tasks). The average score for each type is reported in Table~\ref{nlu-task-table}. 
The evaluation metrics followed those specified in LongBench~\cite{bai2023longbench}, which differ across tasks and are detailed in Supplementary Material D.1.
The results in Table~\ref{nlu-task-table} highlight our model's significant performance advantages over baseline models in four tasks, both for models without training and those with fine-tuning. In summarization tasks, our model also achieved performance comparable to ChatGLM3-6B-$32k$. These experimental outcomes indicate that our model enhances the correlation between tokens with distant relative positions in long contexts through 3D-RPE, resulting in improved performance.

Subsequently, the LEval Benchmark~\cite{an2023eval} was employed. Table~\ref{longeval:leval} reveals that our model, 3D-RPE-LLaMA2-7B-Chat, outperformed LLaMA2-7B-NTK and LongChat-7B-$16K$. Although it did not surpass Vicuna1.5-7B-$16K$ in Quality and CodeU tasks, it excelled in the Coursera, GSM, and TOEFL tasks. 
Additionally, we conducted experiments on LLaMA3-8B-Instruct using a 16k context window with 3D-RPE. The 3D-RPE-LLaMA3-8B-Instruct* showed performance improvements in the Coursera and GSM tasks. While 3D-RPE did not enhance performance in the CodeU, TOEFL, and QuALiTY tasks, there was no significant performance decline either. These experimental results demonstrate the effectiveness of the 3D-RPE method.

\subsection{Long-Sequence Language Modeling}
In Table~\ref{lm-task-table}, we present the perplexity scores for our model, 3D-RPE-LLaMA-2-7B and baseline models on the proof-pile and PG19 test datasets. 3D-RPE-LLaMA-2-7B was fine-tuned from the LLaMA2-7B-Base model using a dataset with a $32k$ context window.
To evaluate performance, we set sequence lengths of $8k$, $16k$, and $32k$. We extended our model's sequence length from $32k$ to $100k$ using the position extending method from PoSE~\cite{zhu2023pose}. 
The results indicate that our method outperforms train-free sequence extending models.
Compared to fine-tuned models, our model shows better performance at $8k$ and $16k$ sequence lengths. 
This suggests that the new positional encoding, 3D-RPE, improves or maintains modeling performance for larger context windows ($32k$) compared to smaller ones ($8k$ and $16k$). 
For the $32k$ and $100k$ tasks, although our model did not surpass LongLoRA-$32k$ and LongLoRA-$100k$, it did outperform LongChat-$32k$ and Activation Beacon. 
 
Notably, our model can further extend from $32k$ to $100k$ without significantly increasing perplexity values, in combination with other train-free extension methods. However, due to its specific attention mechanism, the LongLoRA models cannot be extended beyond their predefined context windows in a train-free manner. For instance, LongLoRA-$32k$ cannot be further extended to $100k$.

\subsection{Ablation Study}
\label{ablation:experiments}
In this section, we conduct ablation studies in this section to explore how 3D-RPE affects the linear interpolation method. We compare position interpolation methods (PI, NTK, and Yarn) with the method that combines 3D-RPE with position interpolation on the LLaMA-2-7B-Base model in a train-free manner. The experimental results can be found in Table~\ref{longeval:leval}. The 3D-RPE-LLaMA2-7B* model with linearly positional interpolation from $8k$ to $16k$ and $32k$, the 3D-RPE approach yields improved results by mitigating the decrease in positional resolution caused by interpolation methods. These results are consistent with the findings of Theorem 1 in Section~\ref{sencond:benefits} presented in this paper.

\section{Conclusion and Future Work}
\label{couclusion}
In this paper, we present a novel rotary position encoding method called 3D Rotary Position Encoding (3D-RPE). Compared to RoPE, we have theoretically proved that 3D-RPE possesses two key advantages: controllable long-term decay and enhanced interpolation resolution. Experimentally, 3D-RPE has demonstrated outstanding performance in long-context Natural Language Understanding. 

In the future, 3D-RPE holds promise as a foundational positional encoding strategy for LLMs, especially in the aspect of modeling long contexts. 
Moreover, given that 3D-RPE encapsulates positional encoding within a three-dimensional framework, it has the potential to integrate with visual data, thereby facilitating an in-depth exploration of its efficacy in synchronizing graphical and textual semantic information.

\medskip

\small
\bibliography{neurips_2024}
\bibliographystyle{plain}

\appendix
\section{Bloch Sphere}
\label{appendix-bs}

\begin{figure}[ht]
\vskip 0.1in
\begin{center}
\includegraphics[width=\columnwidth]{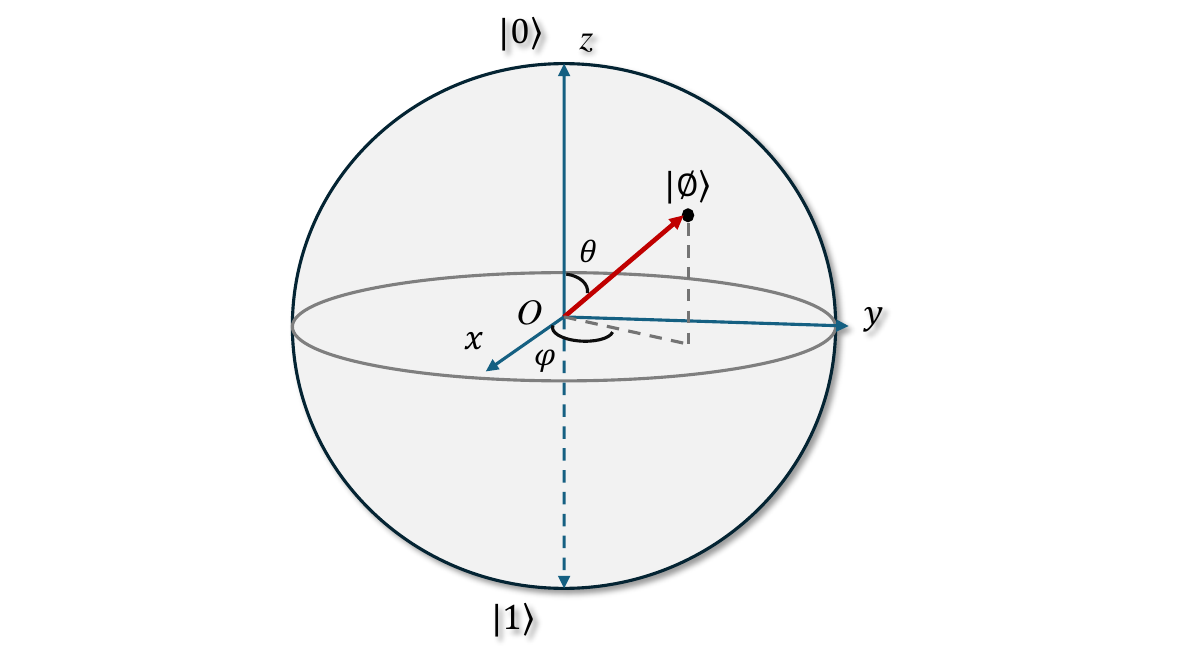}
\caption{A diagram of Bloch Sphere.}
\label{fig:diagram}
\end{center}
\vskip -0.1in
\end{figure}
\textbf{Bloch Sphere}:
3D Rotary Position Encoding (3D-RPE), proposed by us, corresponds to a Bloch Sphere. In this section, we mainly introduce the basic concept of Bloch Sphere, which corresponds to  Eq.~(1) in this paper. 

The Bloch Sphere is a geometric tool to used to represent qubits, typically depicted in a three-dimensional polar coordinate system as a point on the Sphere (see Figure~\ref{fig:diagram}). A single quantum state is represented by the following equation in linear algebra:

\begin{equation}
\label{single_quatum_state}
\ket{\phi}=\alpha
\ket{0}+\beta\ket{1}
\end{equation}

where $\alpha$ and $\beta$ are complex numbers, i.e., $\alpha,\beta\in\mathbb{C}$. According to Euler's formula in complex analysis, the coefficients $\alpha$ and $\beta$ can be reexpressed as:

\begin{equation}
\label{appendix:eq2}
\begin{aligned}
\alpha = a + b\mathrm{i} = r_0(cos(\theta_\alpha)+\mathrm{i}sin(\theta_\alpha))=r_0e^{\mathrm{i}\theta_\alpha}\\
\beta = c + d\mathrm{i} = r_1(cos(\theta_\beta)+\mathrm{i}sin(\theta_\beta))=r_1e^{\mathrm{i}\theta_\beta}
\end{aligned}
\end{equation}
By substituting Eq.~(\ref{appendix:eq2}) into Eq.~(\ref{single_quatum_state}),
the quantum state representation is denoted as:

\begin{equation}
\label{process-of-bs}
\begin{aligned}
\ket{\phi}&=r_0(cos(\theta_\alpha)+\mathrm{i}sin(\theta_\alpha))\ket{0}+r_1(cos(\theta_\beta)+\mathrm{i}sin(\theta_\beta))\ket{1}\\
&=r_0e^{\mathrm{i}\theta_\alpha}\ket{0}+r_1e^{\mathrm{i}\theta_\beta}\ket{1}\\
&=e^{\mathrm{i}\theta_\alpha}(r_0\ket{0}+r_1e^{\mathrm{i}(\theta_\beta-\theta_\alpha)}\ket{1})
\end{aligned}
\end{equation}

$\theta_\alpha$ is the global phase.

Considering the normalization condition $|\alpha|^2+|\beta|^2=1$, we have:

\begin{equation}
|r_0|^2 + |r_1e^{\mathrm{i}(\theta_\beta-\theta_\alpha)}|^2 = r_0^2 + r_1^2|e^{\mathrm{i}(\theta_\beta-\theta_\alpha)}|^2=r_0^2+r_1^2=1
\end{equation}

Given $r_0=cos\frac{\varphi}{2}$, $r_1=sin\frac{\varphi}{2}$, $\theta_\alpha=\theta$ and $\theta_\beta-\theta_\alpha=\theta_1$,
the state $\ket{\phi}$ can be expressed as:

\begin{equation}
\ket{\phi}=e^{\mathrm{i}\theta}(cos\frac{\varphi}{2}\ket{0}+sin\frac{\varphi}{2}e^{\mathrm{i}\theta_1}\ket{1})
\end{equation}

Therefore, the Eq.~(1) of this paper is given out. 

To adapt to the original 2D rotation position encoding (RoPE) of pre-trained LLMs, such as LLaMA models, the global phase $\theta$ is used to model the relative positions between tokens within a chunk, while the rotation angle $\frac{\varphi}{2}$ is used to model the relative positions between tokens across chunks.


\section{Supplementary Material for the Method Section}
\label{appendix:method}
In this section, we mainly introduce the specific implementation of our positional encoding method (3D-RPE), and the formula derivation details of attention score calculation (Eq. (5) of this paper) not detailed in this paper.
\subsection{Implement of 3D-RPE}
\label{appendix-cpe}
In Section 3.1,  we give the general form of 3D Rotary Position Encoding (3D-RPE):
\begin{equation*}
\label{eq:appendix-cpe-general}
\widetilde{\bm{h}}_{j,m} = e^{\mathrm{i}m\theta}(\cos{\varphi_j}\bm{h}_{j, m}^{\perp} + \sin{\varphi_j}\bm{h}_{j, m})
\end{equation*} 
$\cos{\varphi_j}$ and $\sin{\varphi_j}$ are scalar quantities in $\mathbb{R}$. $\bm{h}_{j, m}^{\perp}$ and $\bm{h}_{j, m}$ are shown below:
\begin{equation}
\bm{h}_{j,m}=
\begin{bmatrix}
    h^0 \\
    h^1 \\
    \vdots \\
    h^{d/2-2} \\
    h^{d/2-1} \\
    h^{d/2} \\
    h^{d/2+1}\\
    \vdots \\
    h^{d-2} \\
    h^{d-1}\\
\end{bmatrix}
\quad\quad\quad\quad\quad\quad
\bm{h}_{j,m}^\perp=
\begin{bmatrix}
    -h^{d/2} \\
    -h^{d/2+1} \\
    \vdots \\
    -h^{d-2} \\
    -h^{d-1} \\
    h^0 \\
    h^{1}\\
    \vdots \\
    h^{d/2-2}\\
    h^{d/2-1}\\
\end{bmatrix}
\end{equation}
In the concrete implementation, analogous to RoPE, for each two-dimensional subspace $\mathbb{R}^2$ of $\mathbb{R}^{d}$, we assign angles $\theta_l=base^{-2l/d}$ that vary from high to low frequencies. An equivalent rotation matrix $\mathcal{R}_m^\theta$ is utilized to substitute for $e^{\mathrm{i}m\theta}$:
\begin{equation}
\label{equ:R^d_theta}
\mathcal{R}^\theta_m={
\left[ 
\begin{array}{cccccccc}
\cos m\theta_0 & 0 & \cdots &0 & -\sin m\theta_0 & 0 & \cdots & 0  \\
0 & \cos m\theta_1 & \cdots &0 &0   & -\sin m\theta_1 & \cdots & 0  \\
\vdots & \vdots & \ddots  &\vdots & \vdots &\vdots  & \ddots & \vdots \\
0 &0 &\cdots &\cos m\theta_{d/2-1}  &0 &0 &\cdots & -\sin m\theta_{d/2-1}\\
\sin m\theta_0 & 0 & \cdots &0 & \cos m\theta_0 & 0 & \cdots & 0 \\
0 & \sin m\theta_1 & \cdots &0 &0   & \cos m\theta_1 & \cdots & 0  \\
\vdots & \vdots & \ddots  &\vdots & \vdots &\vdots  & \ddots & \vdots \\
0 &0 &\cdots &\sin m\theta_{d/2-1}  &0 &0 &\cdots & \cos m\theta_{d/2-1}\\
\end{array} 
\right ]}
\end{equation}
Therefore, Eq.(4) of this paper can be transformed to 
\begin{equation*}
\widetilde{\bm{h}}_{j,m} = \mathcal{R}^\theta_m(\cos{\varphi_j}\bm{h}_{j, m}^{\perp} + \sin{\varphi_j}\bm{h}_{j, m})
\end{equation*} 
where $\mathcal{R}^\theta_m$ is a design form equivalent to the rotation matrix in RoPE, mainly re-mapped to correspond to specific application implementations and calculation derivations in LLMs.
In the specific implementation, after the rotary position encoding of LLMs, the long sequence is chunked based on the chunk size $c$. Then, the rotation $\varphi_j$ is set on each chunk, $j$ is the position of chunk.

\subsection{Derivation of Attention for 3D-RPE}
The formula derivation details of attention score calculation(Eq.~(5)) is as follows. 

Since $\bm{h}^\perp=e^{\mathrm{i}\frac{\pi}{2}}\bm{h}=\mathrm{i}\bm{h}$, we could get:
\begin{equation}
\label{cpe_method}
    \begin{aligned}
        \widetilde{\bm{h}}_{j,m} &= e^{\mathrm{i}m\theta}(\mathrm{i}\cos{\varphi_j}\bm{h}_{j, m} + \sin{\varphi_j}\bm{h}_{j, m})\\
        &=e^{\mathrm{i}m\theta}(\mathrm{i}\sin{(\frac{\pi}{2}-\varphi_j})\bm{h}_{j, m} + \cos{(\frac{\pi}{2}-\varphi_j})\bm{h}_{j, m})\\
        &=e^{\mathrm{i}m\theta}e^{\mathrm{i}\frac{\pi}{2}-\varphi_j}\bm{h}_{j, m}
    \end{aligned}
\end{equation}
Let $\bm{q}_{i,m}$=3d-PE$(\bm{q},i,m)$, $\bm{k}_{j,n}$=3d-PE$(\bm{k},j,n)$. 
Taking the real part of the inner product of $\bm{q}_{i,m}$ and  $\bm{k}_{j,n}$ yields:
\begin{equation}
s(\bm{q}_{i,m},\bm{k}_{j,n})=Re[e^{\mathrm{i}(\varphi_i-\varphi_j)}\sum\limits_{l=0}^{d/2-1}e^{\mathrm{i} (m-n)\theta_l}(\bm{q}_{l}\bm{k}_{l}+\bm{q}_{d/2+l}\bm{k}_{d/2+l})]
\end{equation}
which is a function related to both $m-n$ and $\varphi_i-\varphi_j$.

\section{3D Rotary Position Encoding Resolution Enhancement}
In this section, before proving \textbf{Theorem 1}, we first provide the definitions of positional resolution for RoPE, as well as the positional resolution after positional interpolation.
\begin{definition}[\textbf{Positional Interpolation Resolution}]
    Let $\bm{q}_{m+1}$ and $\bm{k}_{m}$ be query state and key state of the $m$-th and ${(m+1)}$-{th} hidden states after RoPE. Given a pre-training length $L_p$, the attention score $a$ is:
\begin{equation}
    a(\bm{q}_{m+1}, \bm{k}_{m}) = \bm{q} \bm{k}^{T} e^{\mathrm{i}\theta}
\end{equation}
The Resolution $\mathcal{E}_{rope}$ corresponding to the initial length $L_p$ is $\mathcal{E}_{rope}=1$. After employing linear interpolation with length $L\geq L_p$, the attention score is:
\begin{equation}
    a(\bm{q}_{m+1}, \bm{k}_m) = \bm{q} \bm{k}^{T} e^{\mathrm{i}\frac{L_p}{L}\theta}
\end{equation}
Note that the Resolution $\mathcal{E}_{rope}$ turns to $\mathcal{E}'_{rope}=L_p/L\leq1$ and decreases as $L$ increases.
\end{definition}
As the resolution decreases, the magnitude of the rotation of attention score becomes smaller, reflecting the extent of positional difference becomes smaller.
Now we give the following theorem, explaining how 3D-RPE mitigates the resolution decreasing in detail.
\begin{theorem}
[\textbf{Chunk Position Encoding Resolution Enhancement}]
For a pre-trained language model with a pre-training length $L_p$ and an extension length requirement of $L$,  employing linear position interpolation extension methods $\mathcal{I}$ based on Rotary Position Encoding (RoPE) can elevate the relative positional resolution from $\mathcal{E}_{rope}$ to $\mathcal{E}_{rope}^\prime$, Let $\mathcal{E}_{3d-rpe}^\prime$ denote the relative positional encoding resolution achieved by the method $\mathcal{I}$ based on 3D-RPE, with chunk size $c>=3$, there is: 
\begin{equation}
   \mathcal{E}_{3d-rpe}^\prime  > \mathcal{E}_{rope}^\prime
\end{equation}
\end{theorem}

\begin{proof}
For 3D-RPE, let the chunk size and chunk number be denoted as $c$ and $n=\lceil L_p/c\rceil$ respectively. Prior to interpolation, the indices within a chunk range from $[0,1,\cdots,c-1]$. Linear interpolation involves evenly distributing the excess $L-L_p$ tokens across $n$ chunks. This results in new indices within the chunk, range from $[0,1,2,\cdots, c'-1]$, where $c'=\lceil L/n\rceil\leq L_p$. So the attention score of $\bm{q}_{i,m+1}$ and $\bm{k}_{i,m}$ based on 3D-RPE after interpolation is:
\begin{align*}
a_{3d-rpe} &= \bm{q}\bm{k}^{T} e^{\mathrm{i}\theta}e^{\mathrm{i}(\varphi_i - \varphi_i)} \\
 &= \bm{q}\bm{k}^{T}e^{\mathrm{i}\theta }
\end{align*}
The resolution of relative position for 3D-RPE is:
\begin{equation*}
    \mathcal{E}_{3d-rpe}^\prime = 1
\end{equation*}
For special cases $\bm{q}_{(i+1, 0)}$ and $\bm{k}_{(i, c' - 1)}$:
\begin{equation}
\mathcal{E}_{3d-rpe}^\prime\geq c' - 1 + \frac{(\varphi_{i+1}-\varphi_i)}{\theta}> c' - 2 \geq 1
\end{equation}
where $(\varphi_{i+1}-\varphi_i)/\theta \geq -1/10000 > -1$. As long as $c'\geq 3$, there is $\mathcal{E}_{3d-rpe}^\prime\geq1>\mathcal{E}_{rope}^\prime=L_p/L$. Under normal case, the chunk size $c$ is not set to a very small number, hence $c'\geq3$ is certainly established; moreover, for different interpolation lengths $L$, we need to configure a varying number of chunks $n$, such that $c'=\lceil L/n\rceil\leq L_p$.
\end{proof}

\section{Experimental Supplementary Materials}
\label{Experience:Detail}
\subsection{Evaluation Metrics}
This section mainly presents the utilization of evaluation metrics for a total of 16 tasks from the LongBench.
\begin{table}[h] 
\centering
\begin{tabular}{|c|c|} 
\hline \textbf{Dataset} & \textbf{Metric} \\
\hline 
Narrative QA & \textnormal{F1\_Score} \\
Qsper & \textnormal{F1\_Score} \\
MultiFieldQA-En & \textnormal{F1\_Score} \\
Hotpot QA & \textnormal{F1\_Score}\\
2WikiM QA & \textnormal{F1\_Score}\\ 
Musique & \textnormal{F1\_Score} \\ 
GovReport& \textnormal{Rouge\_Score} \\
QMSum & \textnormal{Rouge\_Score} \\
MultiNews & \textnormal{Rouge\_Score} \\
Trec & \textnormal{Classification\_Score} \\
Trivia QA & \textnormal{F1\_Score} \\
SAMsum & \textnormal{Rouge\_Score} \\
PassageRetrieval-En & \textnormal{Retrieval\_Score} \\
\textnormal{Passage Count} & \textnormal{Count\_Score} \\  
Lcc & \textnormal{Code\_Sim\_Score} \\ 
RepoBench-P & \textnormal{Code\_Sim\_Score} \\
\hline
\end{tabular} 
\end{table}

\subsection{Details of Experimental Results}
This section mainly presents the performance of all tasks corresponding to each type of experiment in LongBench. These experimental results are reported in Table~\ref{Zero-shot-tasks}.

\begin{table}[ht]
\caption{Comparison of Experimental Performance on Different Datasets for Various Tasks in LongBench, Using Baseline Models Provided by LongBench. 3D-RPE-LLaMA2-7B is our model.}
\label{Zero-shot-tasks}
\vskip 0.1in
\begin{center}
\begin{sc}
\begin{tabular}
{p{5cm}p{2.4cm}p{2cm}p{3.2cm}} \hline
\toprule
\textnormal{Single-Document QA} & \textnormal{Narrative QA} & \textnormal{Qasper} & \textnormal{MultiFieldQA-En}  \\ 
\hline
\textnormal{LLaMA2-7B-Chat-$4k$}  & 18.7        &    19.2       &      36.8      \\
\textnormal{LongChat-v1.5-7B-$32k$}  &   16.9   & 27.7  &  41.4 \\
\textnormal{InternLM-7B-$8k$}    & 12.1      &    16.7       &      23.4        \\
\textnormal{Vicuna-v1.5-7B-$16k$}  &   19.4   &    26.1       &    38.5     \\
\textnormal{LongLora-$16k$}  &   19.8   &    29.1       &    37.1      \\
\textnormal{3D-RPE-LLaMA2-7B(our)} &   40.56   &    41.35 &    60.3    \\
\end{tabular}
\begin{tabular}{p{5cm}p{2.2cm}p{2cm}p{2.8cm}} \hline
\toprule
\textnormal{Multi-Document QA} & \textnormal{Hotpot QA} & \textnormal{2WikiM QA} & \textnormal{Musique}  \\ 
\hline
\textnormal{LLaMA2-7B-chat-$4k$}              & 25.4       &    32.8       &      9.4     \\
\textnormal{LongChat-v1.5-7B-$32k$}  &  31.5 & 20.6 & 9.7  \\
\textnormal{InternLM-7B-$8k$}    & 28.7      &    22.8       &    9.0       \\
\textnormal{Vicuna-v1.5-7B-$16k$}  &   25.3   &    20.8       &   9.8       \\
\textnormal{LongLora-$16k$}  &  37.01   &    30.26       &    17.14     \\
\textnormal{3D-RPE-LLaMA2-7B(our)}  &   \textbf{62.49}   &    \textbf{58.80}      &    \textbf{59.01}    \\
\end{tabular}

\begin{tabular}{p{5cm}p{2.2cm}p{2cm}p{2.8cm}} \hline
\toprule

\textnormal{Summarization} & \textnormal{GovReport} & \textnormal{QMSum} & \textnormal{MultiNews} \\ 
\hline
\textnormal{LLaMA2-7B-chat-$4k$}              & 27.3       &    20.8       &      25.8   \\
\textnormal{LongChat-v1.5-7B-$32k$}  & 30.8 & 22.7 & 26.4 \\
\textnormal{InternLM-7B-$8k$}    & 9.7      &    15.9       &    22.8     \\
\textnormal{Vicuna-v1.5-7B-$16k$} &   27.9   &    22.8       &   27.2      \\
\textnormal{LongLora-$16k$}  &   31.53   &    24.13       &    27.74   \\
\textnormal{3D-RPE-LLaMA2-7B(our)}  &  \textbf{32.01}   &   \textbf{25.3}   &   \textbf{29.68}  \\

\end{tabular}
\begin{tabular}{p{5cm}p{2.2cm}p{2cm}p{2.8cm}} \hline
\toprule
\textnormal{Few-shot Learning} & \textnormal{Trec} & \textnormal{Trivia QA} & \textnormal{SAMSum} \\ 
\hline
\textnormal{LLaMA2-7B-chat-$4k$}              & 61.5       &    77.8       &      40.7      \\
\textnormal{LongChat-v1.5-7B-$32k$}  & 63.5 & 82.3 & 34.2 \\
\textnormal{InternLM-7B-$8k$}    & 52.0      &    77.8       &    21.2      \\
\textnormal{Vicuna-v1.5-7B-$16k$}  &   71.5   &    86.2       &   40.8      \\
\textnormal{LongLora-$16k$}  &   63.5   &    85.69  &  \textbf{41.88}    \\
\textnormal{3D-RPE-LLaMA2-7B-$16k$(our)}   &  \textbf{89.50}  &  \textbf{90.00} &   40.00 \\
\end{tabular}
\small
\begin{tabular}{p{5cm}p{3cm}p{3.5cm}p{2.8cm}} \hline
\toprule
\textnormal{Synthetic Tasks} & \textnormal{Passage Count} & \textnormal{PassageRetrival-En}  \\ 
\hline 
\textnormal{LLaMA2-7B-chat-$4k$}  & 2.1       &    9.8        \\
\textnormal{LongChat-v1.5-7B-$32k$}  & 1.0 &    \textbf{30.5} \\
\textnormal{InternLM-7B-$8k$}    & 3.0      &    6.0         \\
\textnormal{Vicuna-v1.5-7B-$16k$}  &   \textbf{6.5}   &    4.5           \\
\textnormal{LongLora-$16k$}  &   3.61       &    29.75        \\
\textnormal{3D-RPE-LLaMA2-7B-16k(our)}  &   4.0   &   14.5          \\
\end{tabular}
\begin{tabular}{p{5cm}p{3cm}p{3.5cm}p{2.8cm}} \hline
\toprule
\textnormal{Code Completion} & \textnormal{Lcc} & \textnormal{RepoBench-P}  \\ 
\hline
\textnormal{LLaMA2-7B-chat-$4k$}   & 52.4    &    43.8     \\
\textnormal{LongChat-v1.5-7B-$32k$}  & 53.0 & 55.3 \\
\textnormal{InternLM-7B-$8k$}    & 44.1     &    28.8  \\
\textnormal{Vicuna-v1.5-7B-$16k$}  &   51.0   &    43.5       \\
\textnormal{LongLora-$16k$}  &   57.61   &    54.45        \\
\textnormal{3D-RPE-LLaMA2-7B-$16k$(our)}  &   \textbf{79.10}   &   \textbf{73.90}   \\
\bottomrule
\end{tabular}
\end{sc}
\small
\end{center}
\vskip -0.2in
\end{table}
\medskip
\small

\end{document}